\documentclass[11pt,letterpaper]{article}

\usepackage[margin=1in]{geometry}
\usepackage{amsmath,amssymb,amsthm,mathtools}
\usepackage{microtype}
\usepackage{booktabs}
\usepackage{enumitem}

\usepackage{xcolor}
\usepackage[backref=page,colorlinks,citecolor=blue,linkcolor=BrickRed]{hyperref}
\usepackage{cleveref}
\setlist[itemize]{topsep=3pt,itemsep=2pt,parsep=0pt,leftmargin=1.4em}
\setlist[enumerate]{topsep=3pt,itemsep=2pt,parsep=0pt,leftmargin=1.6em}
\usepackage{algorithm}
\usepackage{algpseudocode}
\usepackage{xspace}
\newtheorem{theorem}{Theorem}[section]
\newtheorem{lemma}{Lemma}[section]

\newtheorem{corollary}{Corollary}[section]

\theoremstyle{definition}

\newtheorem{remark}[theorem]{Remark}

\newcommand{\snote}[1]{\textcolor{blue}{Sahil: #1}}
\DeclareMathOperator{\VC}{VC}

\DeclareMathOperator{\Reg}{Reg}

\DeclareMathOperator{\clip}{\mathrm{clip}}

\newcommand{\IGNORE}[1]{}
\newcommand{\E}{\mathop{\mathbb{E}}}
\newcommand{\calD}{\mathcal{D}}
\newcommand{\calH}{\mathcal H}
\newcommand{\calF}{\mathcal F}

\newcommand{\calA}{\mathcal A}
\newcommand{\calX}{\mathcal X}
\newcommand{\calY}{\mathcal Y}

\newcommand{\bbE}{\mathbb E}
\newcommand{\bbP}{\mathbb P}

\newcommand{\eps}{\varepsilon}
\newcommand{\PROD}{\ensuremath{\mathsf{PROD}}\xspace}
\newcommand{\ChainedPrediction}{\ensuremath{\mathsf{ChainedPrediction}}\xspace}
\newcommand{\Sample}{\mathsf{Preview}}
\newcommand{\Real}{\mathsf{Real}}
\newcommand{\tO}{\widetilde O}

\newcommand{\labelthis}[1]{%
    \refstepcounter{equation}%
    \tag{\theequation}%
    \label{#1}%
}
\definecolor{toc}{RGB}{13,55,174}	%dark blue smth
\usepackage{hyperref}				%for links in refs
\hypersetup{
colorlinks=true,
citecolor=toc,
filecolor=black,
linkcolor=toc,
urlcolor=toc
}

\usepackage{natbib}
\title{Adversarial Online Classification with a Preview}
\author{Roi Livni\footnote{School of Electrical \& Computer Engineering, Tel Aviv University. Email: roi.livni@mail.huji.ac.il.} \and  Sahil Singla\footnote{School of Computer Science, Georgia Institute of Technology, Atlanta, GA, USA. Email: ssingla@gatech.edu. Supported in part by NSF awards CCF-2327010 and CCF-2440113.}
}

\date{} %\today}

\begin{document}
\maketitle
\vspace{-1.3em}

\begin{abstract}
Worst-case online classification is governed by sequential complexity, such as Littlestone dimension, and can be impossible even for statistically simple classes, such as thresholds of VC dimension one. We study a \emph{preview model} in which an oblivious adversary fixes an entire labeled sequence of length $T$, a uniformly random subset of size $pT$ is revealed before prediction begins, and the remaining $(1-p)T$ examples are then presented in their original adversarial order.

\smallskip
 
Against the best full-sequence hypothesis evaluated on the unrevealed examples, we characterize the dependence on the preview rate $p$: for binary classes of VC dimension $d$, the optimal excess loss is $\Theta(d/p+\sqrt{dT})$, up to the trivial cap at $T$; for multiclass classes we obtain the corresponding $\widetilde O(d_{\rm DS}/p+\sqrt{d_{\rm Nat}T})$ bound with no dependence on the number of labels.  Thus a random preview can replace worst-case sequential complexity by classical statistical dimensions without randomizing the online order. To achieve the sharp binary bound, our \ChainedPrediction algorithm uses an online analogue of chaining, implemented as a multiscale aggregation algorithm rather than only as an analytic argument.

\end{abstract}

\vspace{.4em}
{\small\tableofcontents}

\newpage

\section{Introduction}\label{sec:intro}

PAC learning and online learning are two well-studied ways to formalize prediction from examples. In the PAC model, examples are drawn i.i.d.\,from an unknown distribution. In the online model, examples arrive sequentially, may be chosen adversarially, and the learner must predict before seeing the label.  The two models differ fundamentally in both their combinatorial limitations and their algorithmic tools.

More formally, let $\calX$ be an instance space, let $\calY$ be a label space, and let $\calH\subseteq \calY^{\calX}$ be a hypothesis class.
In the PAC model, an unknown distribution $\calD$ over $\calX\times\calY$ is fixed, and the goal is to learn a hypothesis whose risk is close to $\inf_{h\in\calH}\Pr_{(x,y)\sim\calD}\{h(x)\ne y\}$.
For binary classification, $\calY=\{0,1\}$, the fundamental theorem of statistical learning says that  the sample complexity is governed by the \emph{VC dimension} of the hypothesis class  \cite{VapnikChervonenkis1971,BlumerEhrenfeuchtHausslerWarmuth1989,Hanneke2016}. 
For multiclass classification the analogous picture is more recent: in the usual $\eps$-excess sample-complexity bounds, the DS dimension controls the leading $1/\eps$ term, while the Natarajan dimension controls the $1/\eps^2$ term, up to logarithmic factors \cite{Pabbaraju2026,CohenErezHannekeKorenMansourMoranZhang2025,BrukhimCarmonDinurMoranYehudayoff2022,DanielySabatoBenDavidShalevShwartz2015,Natarajan1989}.

In contrast to the PAC model, the online learning setting considers a learner that receives a sequence of examples and must predict their labels sequentially \cite{Littlestone1988}. These examples need not be generated stochastically, but instead, can be arbitrary and even chosen adversarially. 
The performance of the learner is then compared to the performance of a fixed hypothesis chosen in hindsight \cite{BenDavidPalShalevShwartz2009,RakhlinSridharanTewari2015}.
Online learning can be much harder than PAC learning, and learnability is no longer controlled by the VC dimension. The relevant parameter is instead the \emph{Littlestone dimension} \cite{Littlestone1988, BenDavidPalShalevShwartz2009}. A classical example that demonstrates the separation is for the hypothesis class of \emph{thresholds}:
        $\calH=\{x\mapsto {\bf 1}\{x\le \theta\}:\theta\in[0,1]\}$.
This class has VC dimension one and is statistically one of the simplest nontrivial classes. However, its Littlestone dimension is infinite: an adversary can play a binary search game where after each prediction, the next point is chosen in the remaining interval. By continuing this binary search for an arbitrary depth, the adversary obtains realizable sequences on which any learner makes arbitrarily many mistakes.

A common way to escape the threshold obstruction is to randomize the arrival process. If examples are i.i.d., or if an adversarially chosen data set is revealed in random order, then prefixes carry information about suffixes. 
In such settings, prefix-based learning becomes meaningful: past examples give useful information about future examples, and online-to-batch conversions formalize this intuition in the i.i.d.\ case \cite{CesaBianchiConconiGentile2004}. Recent work shows that in the random-order model online classification is characterized by VC dimension rather than Littlestone dimension \cite{BernasconiCelliColiniBaldeschiFuscoLeonardiRusso2025}. But random order changes the online information pattern: the learner sees a representative prefix before predicting a representative suffix. Under adversarial order, the first half of the sequence may bear little or no resemblance to the second half.

\paragraph{The Preview Model.} 
In this work, we study a different intervention, taken from sample-augmented online algorithms, where a random sample of an adversarial input is revealed before the online part begins \cite{KaplanNaoriRaz2020,KaplanNaoriRaz2022,ArgueFriezeGuptaSeiler2022}. 
We ask what this information pattern buys in online classification. The adversary still fixes the entire labeled sequence. However, before the online phase begins, the learner receives a random \emph{preview}: exactly $pT$ indices are chosen uniformly at random and revealed together with their positions, instances, and labels.  The unrevealed indices are then served online in their original adversarial order. The important point is that the preview is not a distributional assumption about the arrival order. The learner sees some examples from the future, but must still handle the remaining examples in the order chosen by the adversary. Thus the preview gives global information about the adversarial sequence, but it does not make the online prefix representative of the online suffix. The parameter $p$ measures how much of the sequence is previewed. If $p\approx 1/T$, the learner sees only constantly many examples, and the model is close to fully adversarial online learning. If $p=\Theta(1)$, the learner sees a constant fraction of the sequence before the online phase, and one might hope to recover statistical rates. The main question is what happens in the intermediate regime. For example, when $p=T^{-1/2}$, does binary classification behave like a VC-dimensional statistical problem, or like an adversarial online problem governed by Littlestone dimension?

We write $\Sample\subseteq[T]$ for the revealed (sampled) indices and $\Real=[T]\setminus\Sample$ for the indices predicted (real) online. For $A\subseteq[T]$, write 
\[L_A(h) :=\sum_{t\in A}{\bf 1}\{h(x_t)\ne y_t\}.\] 
Since the examples in $\Sample$ are observed before prediction begins, the learner is evaluated only on $\Real$. The comparator is the best hypothesis for the whole labeled sequence:
fix $h^\star\in\arg\min_{h\in\calH}L_{[T]}(h)$ under a fixed tie-breaking rule chosen before the preview, and benchmark the learner against $L_{\Real}(h^\star)$. Since $h^\star$ is fixed before the preview is drawn, $L_{\Real}(h^\star)$ concentrates around $(1-p)L_{[T]}(h^\star)$. Our results show that a small preview can turn the sequential-complexity obstruction of adversarial online learning into statistical-complexity bounds, without randomizing the online arrival order.

\subsection{Our Results}
We now outline our main results and their implications. Our first result concerns binary classification. It shows that a random preview removes the Littlestone-dimension obstruction and leaves a bound controlled only by VC dimension, even though the online examples still arrive in adversarial order.

\begin{theorem}[Sharp binary bound]\label{thm:binary-upper}
Let $\calH\subseteq\{0,1\}^{\calX}$ have VC dimension $d\ge1$.  For every $0<p\le1/10$, 
there is an improper randomized learner $\calA$ such that, for every fixed labeled sequence with optimum $ h^\star\in\arg\min_{h\in\calH} L_{[T]}(h)$,
\[
        \bbE\big[L_{\Real}(\calA)-L_{\Real}(h^\star)\big]
        \le
        C\min\left\{T,\frac dp+\sqrt{dT}\right\}.
\]
Conversely,  there exists a binary class $\calH$ with $\VC(\calH)\le d$ such that every randomized learner $\calA$ has 
        $\bbE\big[L_{\Real}(\calA)-L_{\Real}(h^\star)\big]
        \ge
        c\min\left\{T,\frac dp+\sqrt{dT}\right\}$.
\end{theorem}

This bound has two terms. The $\sqrt{dT}$ term is the usual statistical fluctuation term and is unavoidable even for random labels on a shattered set. The $d/p$ term is the price of discovering enough of the adversarial sequence from the preview; threshold-tree lower bounds show that this term is necessary even when there is no label noise. Thus the rate is sharp up to constants for binary classification.  An additional feature of
the upper-bound algorithm is that it uses only the positions and features
of the preview examples, not their labels. Moreover, in the binary case the same rate also holds against the stronger preview-dependent comparator 
\(\min_{h\in H} L_{\mathrm{Real}}(h)\):  the full-sequence optimum and the Real-optimum differ by only \(O(\sqrt{dT})\) in expectation; see Remark~\ref{rem:real-comparator}.

%The dependence in \Cref{thm:binary-upper} is tight, by two standard constructions. The $\sqrt{dT}$ term is the usual statistical fluctuation term and is forced by random labels on a $d$-point shattered set. The $d/p$ term is the price of discovering enough of the adversarial sequence from the preview: the Littlestone threshold-tree construction leaves, on each coordinate, an expected $\Omega(1/p)$ unresolved suffix of a binary-search path \cite{Littlestone1988}. Thus the rate is sharp up to universal constants for binary classification. An additional feature of the binary upper-bound algorithm is that it uses only the positions and features of the preview examples, not their labels.

It is instructive to view the theorem as exhibiting a \emph{saturation threshold} in the usefulness of the preview.
Below the saturation point $p \approx \sqrt{d/T}$, in the range where the preview term is active, the regret is governed by $d/p$: increasing the preview rate by a factor correspondingly decreases the adversarial-order cost. Once
        $p\gtrsim \sqrt{d/T}$,
the term $d/p$ is dominated by $\sqrt{dT}$, and the regret reaches the statistical scale $O(\sqrt{dT})$. At this point the bottleneck is no longer the adversarial order but the ordinary agnostic fluctuation.

Thus \Cref{thm:binary-upper} gives a sharp interpolation between the adversarial and statistical regimes. For example, for constant $d$ and $p=T^{-1/2}$, our bound is $O(\sqrt T)$. A preview-only learner that trains one fixed classifier on the $pT$ revealed examples and ignores online labels instead gives $\widetilde O(\sqrt{dT/p})$, which is $\widetilde O(T^{3/4})$ for constant $d$ and $p=T^{-1/2}$, while a learner that ignores the preview may suffer regret governed by Littlestone dimension. Surprisingly, the binary upper-bound algorithm in \Cref{thm:binary-upper} uses only the positions and features of the preview examples, not their labels.

We next outline the main result in the multiclass classification model, where the label space may be infinite. For a class $\calH$ let $d_{\rm DS}(\calH)$ and $d_{\rm Nat}(\calH)$ denote its DS and Natarajan dimensions, respectively. It is known  that $d_{\rm Nat}(\calH)\le d_{\rm DS}(\calH)$ always.

\begin{theorem}[Multiclass upper bound]\label{thm:multiclass}
Let $\calH\subseteq\calY^{\calX}$ have DS dimension $D=d_{\rm DS}(\calH)$ and Natarajan dimension $N=d_{\rm Nat}(\calH)$.  In the  preview model, there is an improper randomized learner $\calA$ such that, for every fixed labeled sequence with optimum
       $ h^\star\in\arg\min_{h\in\calH} L_{[T]}(h)$, we have
\[
        \bbE\big[L_{\Real}(\calA)-L_{\Real}(h^\star)\big]
        ~\le~
        \tO\!\left(\frac Dp+\sqrt{NT}\right),
\]
where the hidden factors are polylogarithmic in $D, T$, and $1/p$, but are independent of $|\calY|$.  
\end{theorem}

This dependence is tight up to polylogarithmic factors: the \(\sqrt{NT}\) term is forced by the binary lower bound on a Natarajan-shattered set, while the \(D/p\) term is forced by the standard DS/list-learning obstruction, or equivalently by planting \(D\) independent preview-discovery components whose relevant labels cannot be identified before they are sampled.

The multiclass result separates the two roles played by the preview. The DS term $D/p$ is the cost of using labeled preview examples to find the relevant labels in a possibly infinite label space. Once this local label space has been reduced, the online aggregation cost is controlled by the Natarajan dimension, giving the $\sqrt{NT}$ term. In particular, the bound has no dependence on $|\calY|$.

\IGNORE{\color{red}\snote{defer to discussion or remove?}
More broadly, our results show that the preview model is not tied to the binary setting. Although extending to multiclass classification requires substantially richer combinatorial notions and new algorithmic ingredients, the central ideas of our framework continue to apply: preview samples reduce the effective uncertainty faced by the learner, and the online labels are exploited to progressively refine the remaining hypotheses. This yields a unified perspective on how preview information can be incorporated into sequential learning across different prediction settings, with the appropriate complexity measures determining the resulting guarantees.
}

\subsection{Our Techniques}

A central feature of our approach is that it leverages the statistical nature of the preview  while simultaneously retaining the robustness of online learning algorithms, without requiring the entire data sequence to be stochastic. 
A natural baseline is to treat the preview as a uniform sample from the fixed sequence and apply standard statistical learning methods.
While this strategy can improve upon worst-case online learning guarantees, its performance is fundamentally limited by the statistical accuracy attainable from the preview  alone. 
For a binary class of VC dimension $d$, a preview-only learner has per-example excess on the order of $\widetilde O(\sqrt{d/(pT)})$, which corresponds to total online regret $\widetilde O(\sqrt{dT/p})$ over $T$ rounds.

This baseline misses the threshold behavior suggested by our main theorem.  The preview-only rate $\widetilde O(\sqrt{dT/p})$ reaches the statistical scale $\sqrt{dT}$ only when $p=\Theta(1)$. By contrast, our bound reaches this scale already when $p\gtrsim\sqrt{d/T}$.  Thus the main algorithmic task is not merely to use the preview as a larger training set; it is to use the preview to organize the remaining online learning problem so that the labels revealed online continue to
drive down uncertainty.

\paragraph{Covering the hypothesis class.}
A stronger baseline uses the preview to build a finite \emph{trace cover} of the hypothesis space. Related public-sample cover ideas appear in semi-private learning \cite{alon2019limits, beimel2013private}, where the learner is required to satisfy privacy while having access to a public sample, analogous in spirit to the public preview available here. The idea is to use the preview to choose a finite subset of hypotheses that \emph{approximate} the original class over the preview. 
Choosing one representative for each trace of $\calH$ on the preview gives a finite class of size controlled by Sauer--Shelah. The approximation error is the number of unrevealed indices on which two hypotheses can disagree despite agreeing on the preview; a VC empty-range argument bounds this by roughly $O(d\log(pT/d)/p)$.
Overall, this proper trace-cover approach gives
\[
        O\!\left(\frac{d\log(pT/d)}{p}
        +
        \sqrt{dT\log(pT/d)}\right),
\]
where the first term arises from the quality of the approximation, while the second term stems from the size of the cover bounded by Sauer--Shelah. 
Thus, even up to logarithmic factors, this already uses the preview and online feedback in a nontrivial way.  The sharp binary theorem then removes the two logarithmic losses: one-inclusion removes the logarithm in the approximation term, and sparse online chaining removes the log-size factor in the aggregation term.

%Overall, this approach yields a learning rate of 
%\[ O\left( \frac{d\log (pT/d)}{p} + \sqrt{dT\log (pT)}\right),\] 
%These techniques, however, leave a gap between the optimal achievable statistical rates and how to exploit the preview data. We next describe the two main ideas that improve upon this approach and lead to the optimal result.

\paragraph{Constructing sample-efficient improper learners.}
A limitation of the cover-based approach is that it replaces the original hypothesis class with a finite representative subset. Restricting the learner to such a finite subset is analogous to proper learning and is known to be inherently suboptimal in general \cite{bousquet2020proper}. Our first step is therefore to replace this approximation with an improper family of predictors constructed from the preview, for which we can guarantee a strictly better approximation. To this end, we apply the one-inclusion graph algorithm to every possible labeling of the preview induced by the original class and use the resulting predictors. This choice is particularly well suited to our setting because its leave-one-out guarantee on any finite indexed set averages, over the random preview, to an $O(d/p)$ approximation error on the unrevealed indices.

\paragraph{Algorithmic vs.~statistical chaining.}
Replacing the approximating set with an improper construction eliminates the first logarithmic factor.  However, to obtain the sharp rate we must also handle the second logarithmic term.  This term comes from the log-size of the predictor set, which is not changed merely by replacing proper representatives with improper one-inclusion completions.

The key idea in resolving this second challenge is to exploit the structure of the predictor class.  The predictors are constructed from preview subsamples of different sizes, and their quality improves as the subsample grows.  Thus one should not treat all predictors at the finest scale as unrelated experts.  This is analogous to statistical chaining, where one replaces a crude union bound by a multiscale argument that uses correlations between nearby hypotheses.  The difference is that, in our online setting, the multiscale refinement must be implemented by the algorithm itself rather than being used only in the analysis.

Concretely, a comparator induces a path of predictors
        $g_0,g_1,\ldots,g_K$,
where $g_k$ is the one-inclusion completion built from a level-$k$ subsample of the preview.  We decompose the final predictor as
\[
\textstyle        g_K
        =
        g_0+\sum_{k=1}^K(g_k-g_{k-1}).
\]
The learner runs one second-order experts algorithm for the base predictor $g_0$ and one for each correction $g_k-g_{k-1}$.  The aggregation cost of a correction is charged only on rounds where the level-$k$ prediction differs from its parent.  One-inclusion bounds the number of such changes, while Sauer--Shelah bounds the log-size of the expert class at each level.

Related ideas have previously appeared in the literature in partial-feedback settings or when further structure on the experts can be imposed \cite{cesa1999prediction,cesa2017algorithmic}.  In our setting, there is no ambient metric or topological structure on $\calX$ that the algorithm can rely on.  Instead, the hierarchy is created from the preview itself: each level uses a larger, more informative subsample, and each online subroutine learns how to correct the level below it.  This yields the \ChainedPrediction algorithm (see \Cref{alg:hier_prod}) and removes the log-size loss in the aggregation term.

\paragraph{From binary to multiclass.}
The preceding ideas suffice to obtain the optimal guarantees in the binary setting. Extending them to multiclass learning, however, introduces a new challenge. The hierarchy constructed above relies on predictors obtained from the preview, whose complexity grows with the number of possible labelings. If the label space is infinite, this dependence can be infinite even for simple classes, preventing a direct application of the binary construction.

To overcome this obstacle, we exploit recent advances in multiclass PAC and list learning. Rather than explicitly constructing predictors for every possible labeling, we first use labeled preview examples to build a short local menu of candidate labels, and then learn inside this menu using Natarajan dimension. This avoids any dependence on the ambient label-space size $|\calY|$ and replaces trace enumeration by a DS/Natarajan-controlled representation, yielding the desired DS/Natarajan bound up to logarithmic factors independent of $|\calY|$.

\subsection{Discussion}

The preview model keeps the online order adversarial while giving the learner scattered information about the future, and our bounds show that this is enough to replace sequential-complexity parameters by statistical-complexity parameters.
\IGNORE{We introduced the preview model as an intermediate model between the classical statistical and fully adversarial learning paradigms. In this model, the sequence of examples is chosen adversarially, yet the learner is given access to a random fraction of the examples before the online interaction begins. This setting retains the sequential and non-stochastic nature of online learning, while avoiding the full pessimism of the adversarial model, where many learning problems that are easy in the statistical setting become impossible.}
A recurring theme in our proofs is that statistical information alone is not enough; it must be converted into an online algorithmic object. In the offline setting, empirical risk minimization serves as a remarkably general algorithmic principle, and optimal statistical rates can often be obtained by combining ERM with classical empirical process techniques. In contrast, online learning typically requires the explicit design of learning algorithms, together with algorithmic principles such as regularization, stability, optimism, or potential-based updates. Our  preview model brings this distinction into sharp focus: although the preview provides sufficient statistical information to substantially reduce uncertainty, achieving optimal guarantees still requires exploiting the online labels algorithmically, rather than merely evaluating an offline solution.

%This perspective also motivates a discussion of the techniques underlying our results. At a high level, our algorithms combine statistical information extracted from the preview with inherently online algorithmic principles. While the preview allows the learner to substantially reduce the uncertainty it faces, exploiting this information optimally cannot be achieved by simply running an offline learner. Instead, the learner must continually incorporate the revealed online labels, leading to algorithms that are intrinsically sequential.

This distinction becomes particularly apparent when aiming for optimal rates. In statistical learning, chaining is often an analytic tool for replacing a crude union bound by a multiscale argument. The chaining itself plays no algorithmic role. In our setting, however, such arguments cannot be applied in the same manner. Since the guarantees depend on the learner's sequential predictions, the multiscale refinement must be reflected in the algorithm. Our learner aggregates corrections across levels, paying only when a refinement changes the prediction.
The sparse aggregation lemma may be useful beyond the present model. It gives a way to aggregate a hierarchy of predictors while paying for a refinement only on the rounds where it changes the parent prediction. Similar structures arise in privacy, public-sample learning, and structured expert problems; exploring other preview online problems is a natural direction.

%\paragraph{Computational efficiency.}
Our results are information-theoretic. For an arbitrary hypothesis class
$\calH$, the binary algorithm enumerates traces of $\calH$ on preview
subsamples and applies one-inclusion completions to the resulting named
predictors. We do not claim polynomial-time implementability without
additional structure or appropriate oracle access to $\calH$. The main
contribution is the preview information complexity and the online
chaining reduction.

\subsection{Further Related Work}

Several models mitigate the pessimism of fully adversarial online learning by imposing structure on the sequence. These include random-order models \cite{BernasconiCelliColiniBaldeschiFuscoLeonardiRusso2025,garber2020online}, predictable or constrained adversaries \cite{RakhlinSridharanTewari2015,rakhlin2013online}, stochastic data with adversarial corruptions \cite{amir2020prediction}, and smoothed or  easy-data models \cite{haghtalab2024smoothed,hazan2009stochastic,sani2014exploiting}. In contrast, our model keeps the served order adversarial and instead reveals a random preview of the same sequence before the online phase.
The preview assumption is closest to the ``with a sample'' line of work in online algorithms, including secretary and matching problems with a random sample \cite{KaplanNaoriRaz2020,KaplanNaoriRaz2022} and more general sample-augmented online algorithms \cite{ArgueFriezeGuptaSeiler2022,GKL-SODA24,GM-SODA26,GGPS-STOC26}. Our contribution is to bring this to online classification and identify the resulting statistical dimensions.

%Our work establishes a novel intermediate models between \emph{online} and \emph{offline} classification. Several intermediate models that try to mitigate the pessimism of the online setting, with the purely stochastic nature of the offline models have been investigated, such as for example the random order model \cite{BernasconiCelliColiniBaldeschiFuscoLeonardiRusso2025,garber2020online}. Alternative approach considers adversaries that are constrained \cite{RakhlinSridharanTewari2015}, by considering predictable sequence \cite{rakhlin2013online}, stochastic data that is adversarially corrupted \cite{amir2020prediction}, or smoothed analysis \cite{haghtalab2024smoothed} and other forms of ``easy data" \cite{hazan2009stochastic,sani2014exploiting}. In these settings, similar to our work, \cite{hazan2016online, koren2017affine} observes that while in the stochastic setting such deficiencies can be exploited obliviously through ERMs, achieving comparable guarantees in the online setting requires the learner to use an appropriately designed algorithm.

The first step of our binary proof, constructing approximating predictors from the preview, is related to public-sample and semi-private learning. In those settings, the learner has access to both a private sample and a public sample, and the public data can be used to construct a finite approximating family \cite{alon2019limits,beimel2013private}. This analogy is useful, but our goal is different: the approximating family is fed into an online learner, and the remaining labels arrive adversarially.

%The technique we exploit involves, first, the usage of approximating predictors. A similar solution appears in a problem of learning from \emph{private data}. There too, the data is divided into two types: one whose privacy is to be preserved and the other that is public. The solution turns out to exploit a similar idea \cite{alon2019limits, beimel2013private} and indeed the public data is used to construct an approximate set. Surprisingly, in both setups the usage of the further data helps avoid dependency on the Littlestone dimension which turns out to control not only online learning but also private learning \cite{alon2019private,alon2022private,bun2020equivalence}.

The improved bounds arise from an additional \emph{algorithmic chaining} technique. Similar ideas appeared previously but in substantially different framework and with different methods to perform the chaining. In \cite{cesa2017algorithmic}, the authors consider a partial feedback model in a topological space. The topological structure is inherent to the bandit problem, and the algorithm partitions the prediction space into neighborhoods, assigning an expert-advice learner to each neighborhood. In contrast, our setting lacks a natural topological structure. Instead, we construct a hierarchy of increasingly expressive hypothesis classes. Rather than treating the predictors from one level as experts for the next, each level introduces a new hypothesis class that is trained to correct the predictions produced by the preceding level. 
Cesa-Bianchi and Lugosi \cite{cesa1999prediction} use a related chaining idea for next-bit prediction with static experts. Our setting is different because the relevant predictors are generated from random previews at multiple scales, and the learner must aggregate their corrections online.

%\cite{cesa1999prediction} employs an approach which is more similar to ours, but the framework and setup differ substantially. In particular they consider a next-bit prediction task, and they perform the chaining on \emph{static experts} which in our setting leads to a sequence of predictions known in advance, which is analogue to $p=1$ in our setting that requires no additional algorithmic tools.

Another closely related relaxation is transductive online learning, where
the learner is given the entire unlabeled online sequence before prediction
begins \cite{KakadeK05,ChaseHMS25}. Also related is the prediction-augmented model \cite{RT-NeurIPS24}, in which the learner receives predictions about future examples. Our preview instead consists of exact labels on a uniformly random subset of the adversarial sequence, and the unrevealed future instances are not shown before they arrive.

\section{Model and Preliminaries}\label{sec:model}

We consider a fixed class $\calH\subseteq\calY^{\calX}$ of hypotheses over a fixed domain $\calX$ with label space $\calY$. For binary classification we will have $\calY=\{0,1\}$, and in the multiclass section we let $\calY$ be any set, possibly infinite.  

In our setting, we consider the following iterative game between an adversary and a learner. At the beginning of the game, an oblivious adversary chooses a sequence
\[
        ((x_1,y_1),\ldots,(x_T,y_T))\in(\calX\times\calY)^T.
\]
For any index set $I \subseteq [T]$, denote
\[ L_I(h) := \sum_{t\in I} \mathbf{1}\{h(x_t)\ne y_t\}.
\]
After the sequence is fixed, let
        $h^\star \in \arg\min_{h\in\calH} L_{[T]}(h)$
be chosen by a fixed tie-breaking rule independent of the preview.
Assume that $m=pT$ is an integer.  We choose $\Sample\subseteq[T]$ uniformly at random among all subsets of size $m$, and set $\Real=[T]\setminus\Sample$.
Then, the learner observes for every $t\in \Sample$ both $t$ and 
$(x_t,y_t)$. Now the iterative online game proceeds where in each iteration, the indices in $\Real$ arrive in the original order. On $t\in\Real$, the learner sees $x_t$, predicts a label $\hat y_t$, and then observes $y_t$.

If the learner $\calA$ is randomized, $L_{\Real}(\calA)$ denotes its realized number of mistakes on $\Real$. 
In the binary case, our algorithms
often output a number $r_t\in[0,1]$ and then predict $1$ with probability
$r_t$. Since conditional on the history,     $\bbE[{\bf 1}\{\hat y_t\ne y_t\}\mid r_t,y_t]        =
|r_t-y_t|$,  any bound on $\sum_{t\in\Real}|r_t-y_t|$ immediately implies the same bounds for the realized loss
$L_{\Real}(\calA)$.
The learner's objective is to minimize \emph{regret}:
\[
        \bbE\big[L_{\Real}(\calA)-L_{\Real}(h^\star)\big],
\]
where the expectation is over the random preview and the learner's internal randomness.

%\section{Preliminaries and Technical Background}
\subsection{\PROD algorithm for the expert advice setting}
Arguably, the most extensively studied setting in online learning is the \emph{expert advice} model. 
%For binary $\{0,1\}$ losses, this can be viewed as a special case of our framework in which $\mathcal H$ is an arbitrary, finite, class and no samples are observed initially. 
The seminal \emph{Multiplicative Weights} algorithm is known to achieve optimal guarantees in this setting and has become a standard tool for a wide range of subsequent generalizations \cite{AHK-ToC12}. In our framework, it is convenient to employ one of its variants, \PROD, introduced by \cite{CesaBianchiMansourStoltz2007}. \PROD provides a \emph{second-order} regret guarantee and can also accommodate negative losses, both advantages will be useful in our analysis.

In more details, in the expert advice setting, an arbitrary sequence of loss vectors $\ell_1,\ldots,\ell_T \in [-1,1]^N$ arrives sequentially. At each round $t$, the learner chooses a distribution $p_t \in \Delta_N$, then observes $\ell_t$ and incurs loss $p_t\cdot \ell_t$. The goal is to minimize the regret, defined as the difference between the learner's cumulative loss and that of the best expert in hindsight. The following guarantee can be achieved when $p_t$ are chosen via the \PROD algorithm:
\begin{lemma}[\cite{CesaBianchiMansourStoltz2007}]\label{lem:exp}
For every $0 < \eta \leq 1/2$ there exists an online algorithm that given a sequence of $T$  losses $\ell_{t}\in [-1,1]^N$ returns $p_t$, at each iteration $t$, such that for every $i^\star$:
\[
\sum_{t=1}^T \Big( p_t \cdot \ell_t - \ell_t(i^\star)\Big)
\le
\frac{\log N}{\eta}
+
\eta \sum_{t=1}^T \ell_t^2(i^\star).
\]
If the loss sequence is random, the same inequality holds conditionally on the realized sequence of losses, and expectations may be taken afterwards.
% \[
% \E\left[\sum_{t=1}^T \Big( p_t \cdot \ell_t - \ell_t(i^\star)\Big) \right] \le \frac{\log N}{\eta}+ \eta \sum_{t=1}^T\E\left[\ell^2_t(i^\star)\right]
% .\]
\end{lemma}

\subsection{One Inclusion Graphs}
The second ingredient in our proof is the \emph{one-inclusion graph} algorithm \cite{haussler1994predicting}. In the standard PAC setting, this algorithm achieves the optimal order of expected prediction error. Unlike empirical risk minimization (ERM), which learns by selecting a hypothesis from the target class that minimizes the empirical error, the one-inclusion graph algorithm overcomes inherent limitations by allowing \emph{improper predictions}. Namely, its predictions need not be induced by hypotheses in the target class.

While other black-box optimal learners are known, the one-inclusion graph algorithm enjoys an additional guarantee in terms of the \emph{permutation mistake bound}, which is particularly well suited to our setting: Given an unlabeled sample $S$ and a hypothesis $h$, let $S_{[m]\backslash i}^{h}$ be the sample constructed by labeling according to $h$ and excluding example $i$, i.e.
\[S_{[m]\backslash i}^{h} := \{(x_1,h(x_1)),\ldots,(x_{i-1},h(x_{i-1})),(x_{i+1},h(x_{i+1})),\ldots, (x_m,h(x_m))\}.\]

Then, we have the following result
\begin{lemma}[\cite{haussler1994predicting}]\label{lem:oinclusion}
Let $\calH$ be a class with VC-dimension $d$. For any set of unlabeled points $S=\{x_1,\ldots,x_{m+1}\}$ and $h\in \calH$, let $Q(S^{h}_{[m+1]\backslash i},x_i)$ be the prediction of the one inclusion graph, on sample $x_i$ after observing sample $S_{[m+1]\backslash i}^h$. Then: 
\[ \frac{1}{m+1}\sum_{i=1}^{m+1}  \mathbf{1}\{Q(S^h_{[m+1]\backslash i},x_i)\ne h(x_i)\} ~\le~ \frac{2d}{m+1}.\]
\end{lemma}
The above guarantee is particularly appealing in our setting, especially as \Cref{lem:oinclusion} makes no assumptions on the distribution of $S$. We will employ the one-inclusion graph on refined observed sub-samples, like $\Sample$, throughout the algorithm. 

For $I\subseteq[T]$, let
        $\Pi_I(\calH):=\{(h(x_i))_{i\in I}:h\in\calH\}$
be the set of \emph{traces} of $\calH$ on $I$.
For a trace
$a\in\Pi_I(\calH)$, let $Q_I(a,\cdot)$ denote the one-inclusion completion
obtained as follows: for $t\notin I$, run the one-inclusion rule on the
finite indexed set $I\cup\{t\}$, after observing the labels $a$ on the
coordinates in $I$, and predict the hidden coordinate $t$. This prediction
depends only on the indexed sample, the observed trace $a$, and the fixed
one-inclusion orientation.

We write
\[
        Q_I(\calH):=\{Q_I(a,\cdot):a\in\Pi_I(\calH)\}
\]
as a named class indexed by traces. If two traces induce the same function
outside $I$, we still keep the two names distinct. Equivalently, for
$h\in\calH$ and $S_I^h=\{(x_i,h(x_i)):i\in I\}$, the predictor
$Q(S_I^h,\cdot)$ is the named predictor $Q_I(a_h,\cdot)$, where
$a_h=(h(x_i))_{i\in I}$.

Since $Q_I(\calH)$ is indexed by traces, $|Q_I(\calH)|=|\Pi_I(\calH)|$.
Hence, by the Sauer--Shelah lemma, for $|I|\ge d$:
\begin{equation}\label{eq:logQ}
|Q_I(\calH)| ~\le~ \sum_{i=0}^d {|I| \choose i} ~\le~ \left(\frac{e|I|}{d}\right)^d.
\end{equation}
Notice that the above bound on the size of the class already provides guarantee on its prediction error. Applying \Cref{lem:oinclusion} to our setting provides the following bound:
\begin{corollary}\label{cor:oinclusion} Let $\calH$ have VC dimension $d$.
    Let $x_1,\ldots, x_T$ be an arbitrary sequence, let $I$ be a uniformly random subset of $[T]$ of size $m$, and fix $h\in\calH$ independently of the choice of $I$. Then
    \[ \E_{I} \Big[\sum_{t\in [T]\backslash I} \mathbf{1}\Big[Q(S_I^h,x_t)\ne h(x_t)\Big] \Big]\le \frac{2 d T}{|I|}.\] 
 In particular, for every $h\in \calH$:
     \[ \E_{I} \Big[\min_{q\in Q_{I}(\cal H)}\Big[\sum_{t\in [T]\backslash I} \mathbf{1}\left[q(x_t) \ne h(x_t)\right] \Big]\Big]\le \frac{2 d T}{|I|}.\] 
\end{corollary}
Notice that the expectation is taken only with respect to the random choice of $I$, and not over the sequence of examples. Moreover, the algorithm inherits the optimal guarantee of the one-inclusion graph algorithm while avoiding the unnecessary logarithmic factors that arise in more generic reductions.

\section{Binary Classification}\label{sec:binary}
We next present and prove our main result for the binary case, \Cref{thm:binary-upper}. There are two key ingredients that enable us to attain the optimal rates.
The first is to use the observed preview points to construct a finite class of predictors that contains a near-optimal benchmark. The online learning problem is then reduced to competing with the best predictor in this finite class. A natural approach is therefore to run an expert-advice algorithm over the resulting predictor class. However, while the class is optimal from a statistical perspective, its cardinality is on the order of $O(|\Sample|^d)$. Applying a standard expert-advice algorithm directly would therefore incur an aggregation term of order $\widetilde O(\sqrt{dT})$, but with an extra logarithmic factor coming from $\log |Q_{\Sample}(\calH)|=O(d\log|\Sample|)$. To achieve an improved rate, we invoke the algorithmic chaining technique. 

In this section we overview our main algorithm depicted in
\Cref{alg:hier_prod}. We then analyze and provide its regret guarantees in
\Cref{sec:alg}, prove the upper-bound part of \Cref{thm:binary-upper} in
\Cref{prf:binary-upper}, and prove the matching lower bound in
\Cref{sec:binary-lower}.

\subsection{The \ChainedPrediction Algorithm}\label{sec:alg}
We next depict the algorithm that achieves the regret bound guaranteed in \Cref{thm:binary-upper}. Before, we define the algorithm, we will need some further notations. The algorithm considers an increasing chain of index sets
\[I_0 \subset I_{1} \subset \cdots I_{K-1} \subset I_K  = \Sample.\]

In this section we regard $Q_I(\calH)$ as a named class indexed by traces: for every trace
$a\in\Pi_I(\calH)$ we keep one named predictor $Q(a,\cdot)$.  If two traces induce the same
function, we still keep the two names distinct.  For $k=1,\ldots,K$, define
\[
        \pi_{k-1}:Q_{I_k}(\calH)\to Q_{I_{k-1}}(\calH)
\]
by restricting the trace: if $q=Q(a,\cdot)$ for a trace $a\in\Pi_{I_k}(\calH)$, then
\begin{equation}\label{eq:pik}
        \pi_{k-1}(q):=Q(a|_{I_{k-1}},\cdot).
\end{equation}
For a terminal predictor $q\in Q_{I_K}(\calH)$ we also write $\pi_K(q)=q$. More generally, for a terminal predictor $q\in Q_{I_K}(\calH)$, we write $\pi_k(q)$ for its level-$k$ ancestor obtained by repeatedly applying the parent maps. Thus $\pi_{k-1}(\pi_k(q))=\pi_{k-1}(q)$.
In other words, $\pi_{k-1}(q)$ is obtained by forgetting the labels of the trace outside $I_{k-1}$.  Thus if a terminal predictor is induced by a fixed hypothesis $h$, all of its ancestors are induced by the same fixed $h$ on the coarser samples.

% {\color{red}
% Given any such chain,  define for $k=1,\ldots,K$,
% \[
%         \pi_{k-1}:Q_{I_k}(\calH)\to Q_{I_{k-1}}(\calH).
% \]
% Given $q\in Q_{I_k}(\calH)$, choose any $h_q\in\calH$ such that
% $q=Q(S_{I_k}^{h_q},\cdot)$, and set
% \begin{equation}\label{eq:pik}
%         \pi_{k-1}(q):=Q(S_{I_{k-1}}^{h_q},\cdot).
% \end{equation}
% For a terminal predictor $q\in Q_{I_K}(\calH)$ we also write $\pi_K(q)=q$.

% In other words, $\pi_{k-1}(q)$ is the predictor obtained from the same
% underlying hypothesis $h$, but using the cruder sample indexed by
% $I_{k-1}$ rather than the larger sample indexed by $I_{k}$.
% Consequently, $\pi_k$ maps each predictor at level $k+1$ to its
% natural ancestor at level $k$.
% }

We are now ready to introduce the \ChainedPrediction algorithm (\Cref{alg:hier_prod}). At a high level, the   algorithm constructs a hierarchy of expert classes indexed by increasingly refined subsamples. Starting from a coarse subsample, each subsequent level is obtained by revealing additional sample points and considering the corresponding one-inclusion graph. As a result, the expert class at level $k$ is strictly more informative than the one at level $k-1$.
The prediction process mirrors this hierarchy. The lowest level provides a coarse prediction based on limited information, while each higher level is viewed as playing an expert-advice game whose goal is to correct the residual error of the level below it.

\begin{center}
\fbox{
\begin{minipage}{0.92\linewidth}
\textbf{Chained prediction idea.}
A comparator $h$ induces a path of one-inclusion completions
\[
        g_0(h),g_1(h),\ldots,g_K(h),
\]
where level $k$ uses a larger subset of the preview.  The final prediction decomposes as
\[
        g_K ~=~ g_0+\sum_{k=1}^K \Big(g_k-g_{k-1} \Big).
\]
We run one second-order experts algorithm for the base prediction $g_0$ and one for each increment
$g_k-g_{k-1}$.  The increment at level $k$ is charged only on rounds where $g_k$ changes the
prediction of $g_{k-1}$.  One-inclusion bounds the expected number of such changes, and
Sauer--Shelah bounds the log-size of the expert class at that level.
\end{minipage}}
\end{center}

The \ChainedPrediction algorithm has the following regret guarantee:

\begin{lemma}\label{lem:hier_prod_regret}
Let $(x_1,y_1),\ldots, (x_T,y_T)$ be an arbitrary sequence, $\calH$ an hypothesis class with VC dimension $d$, and $\Sample\subseteq [T]$ a subset of indices of size $pT$, chosen uniformly at random. Apply \ChainedPrediction (\Cref{alg:hier_prod}) on the sequence $\Real$ after observing $\Sample$. 
Then for every fixed $h\in\calH$, independent of the preview, and for
$q=Q(S_{\Sample}^h,\cdot)$:
\[
\E\Big[
\sum_{t\in\Real}|r_t-y_t| - \sum_{t\in\Real}|q(x_t)-y_t|
\Big] 
%\E\left[\sum_{t\in \Real} \mathbf{1}|\hat y_t \ne y_t| - \sum_{t\in \Real}\mathbf{1}|q(x_t)\ne y_t|\right]
\le 2\sqrt{dT \log (\beta_0 T)} + 4\sum_{k=1}^K \sqrt{\frac{4ed\log (\beta_k T)}{\beta_{k-1}}} + 5d\sum_{k=0}^K \log (\beta_k T),\]
where the expectation is over the random preview and the algorithm's internal randomness.
\end{lemma}

%\snote{change LHS to $\hat y_t \neq y_t$}

\begin{algorithm}
\caption{\ChainedPrediction}\label{alg:hier_prod}

\begin{algorithmic}[1]
\Require Sample indices $\Sample\subseteq[T]$ of size $pT$ from sequence $x_1,\ldots,x_T$, real indices $\Real=[T]\setminus\Sample$, fractions $\frac{e}{T}\le \beta_0\le  \ldots \le  \beta_{K-1}\le \beta_{K}= \frac{ep}{d}$
\State Set $I_{K}=\Sample$
\For{$k= K-1,\ldots,0$}
    \State Draw $I_k\subset I_{k+1}$ uniformly so that $\frac{e|I_k|}{d}=\beta_k T$.
    \State Let $\pi_{k}$ be as in \Cref{eq:pik}.
\EndFor
\State Initialize uniform distributions $p^k_1$ over $Q_{I_{k}}(\calH)$ for each $k=0,\ldots,K$

\For{each arriving $t\in\Real$ in the original order}
        \State Define prediction for each $0<k\le K$:
        \[
        \widehat y_{k,t}
        \gets
        \sum_{q\in Q_{I_k}(\calH)}
        p^k_t(q)\Big(q(x_t)-\pi_{k-1}(q)(x_t)\Big)
        \]
        \State And for $k=0$:
        \[
        \widehat y_{0,t}
        \gets
        \sum_{q\in Q_{I_0}(\calH)}
        p^0_t(q)q(x_t)
        \]
    \State Set
\[
        r_t=\clip_{[0,1]}\!\left(\sum_{k=0}^K \widehat y_{k,t}\right).
\]
\State Prediction $\hat y_t$ equals $1$ with probability $r_t$ and is $0$ otherwise.
    % \State Aggregate and predict:
    % \[
    % \widehat y_t
    % \propto
    % \clip\!\left(\sum_{k=0}^K \widehat y_{k,t}\right) ,
    % \]
    \State Receive feedback $y_t$
    \State Define signal $s_t$:
    \[
    s_t
    \gets
      \mathrm{sign}\left(
    \sum_{k=0}^K \widehat y_{k,t}-y_t
    \right)
    \]

    \For{$k=1,\ldots,K$} \Comment{Set hierarchical losses}
%    \For{$k=1,\ldots,K$} \% Set hierarchical losses
        \State Define for each $q\in Q_{I_{k}}(\calH)$:  
            \[
            \ell^k_t(q)
            \gets
            s_t\cdot  \left(
            q(x_t)-\pi_{k-1}(q)(x_t)
            \right)\in \{-1,0,1\}
            \]
            \State Update $p^k_{t+1}$ from $p^k_t$ using the Prod update with losses $\ell^k_t$ and
            \[ \eta_k = \min\left\{1/2, \sqrt{\frac{d\beta_{k-1} \log (\beta_k T)}{4e}}\right\}.\]
        \EndFor
    \State Define for each $q\in Q_{I_0}(\calH)$  
        \[\ell^0_t(q)
        \gets s_t q(x_t)
        \]
        \State Update $p^0_{t+1}$ from $p^0_t$ using the Prod update with losses $\ell^0_t$ with 
        \[ \eta_0 = \min\left\{1/2,\sqrt{\frac{d \log (\beta_0 T)}{T}}\right\}.\]
    \EndFor
\end{algorithmic}
\end{algorithm}
\begin{proof}
Since $y_t\in\{0,1\}$, projection onto $[0,1]$ cannot increase distance to $y_t$, we have     $|r_t-y_t|
        \le
        \left|\sum_k \widehat y_{k,t}-y_t\right|$.
This implies
\begin{align*}
\Reg
:=
\sum_{t\in\Real}|r_t-y_t|
-
\sum_{t\in\Real}|q(x_t)-y_t| 
~\le~ \sum_{t\in\Real} \left|\sum_{k=0}^K \hat y_{k,t}-y_t\right| - \sum_{t\in\Real} |q(x_t)-y_t|.
\end{align*}

From telescoping and the fact that $\pi_{K}(q)=q$, this implies
\[
\Reg \leq \sum_{t\in\Real} \left|\sum_{k=0}^K \hat y_{k,t}-y_t\right| - \sum_{t\in\Real} \left|\pi_0(q)(x_t) + \sum_{k=1}^K \Big(\pi_{k}(q)(x_t)-\pi_{k-1}(q)(x_t) \Big) -y_t\right|
.
\]

Next we claim that for every two vectors $a,b\in \mathbb{R}^{K+1}$ and $y$, the following inequality holds 
\begin{equation}\label{eq:ellconvex}
\left|\sum_{k=0}^K a_k-y\right|-\left|\sum_{k=0}^K b_k - y\right| \le \mathrm{sign}\left(\sum_{k=0}^K a_k-y\right)\sum_{k=0}^K  (a_k-b_k).
\end{equation}
Indeed, the inequality is immediate by observing that $\ell(\vec{x})=\left|\sum_{k=0}^K x_k-y\right|$ is a convex function, hence the inequality $\ell(a)-\ell(b)\le \partial \ell(a)^\top(a-b)$ holds, and
the result follows by taking any subgradient of the absolute-value function at the point $\sum_k a_k-y$.

Applying \Cref{eq:ellconvex} with $a= (\hat y_{0,t}, \hat y_{1,t},\ldots, \hat y_{K,t})$ and\\ $b=\left(\pi_0(q)(x_t),\pi_1(q)(x_t)-\pi_0(q)(x_t),\ldots, \pi_K(q)(x_t)-\pi_{K-1}(q)(x_t)\right)$, we obtain by definition of $s_t$:
\begin{align*} \Reg&\le \sum_{t\in\Real}\left(s_t\cdot\big[ \hat y_{0,t}-\pi_0(q)(x_t) \big] + \sum_{k=1}^K s_t\cdot \big[ \hat{y}_{k,t}- (\pi_{k}(q)(x_t)-\pi_{k-1}(q)(x_t)) \big] \right)\\
&= \sum_{t\in\Real}\left(s_t\cdot \hat y_{0,t}-\ell^0_t(\pi_0(q))  + \sum_{k=1}^K s_t\cdot \hat{y}_{k,t}- \ell_t^k(\pi_{k}(q)) \right)  \qquad \textrm{ using } \pi_{k-1}(\pi_k(q))=\pi_{k-1}(q).
\end{align*}
Next, we note that by definition of $\ell_t^0$: 
\[s_t\hat y_{0,t} ~=~ s_t\sum_{q\in Q_{I_0}(\calH)}p^0_t(q) q(x_t) ~=~ \sum_{q\in Q_{I_0}(\calH)}p^0_t(q) (s_t q(x_t)) ~=~ p^0_t\cdot \ell_t^0,\]
and similarly
$s_t \hat y_{k,t} = p_t^k\cdot \ell_t^k.$ 
Hence,
\begin{align*}
\Reg 
~\le~ \sum_{t\in\Real}\left( p_{t}^{0}\cdot \ell_t^0-\ell_t^0(\pi_0(q))+\sum_{k=1}^K \big[ p_t^{k}\cdot \ell_t^k - \ell_t^k(\pi_{k}(q)) \big]\right)
%&= \sum_{t}\left( p_{t}^{0}\cdot \ell_t^0-\ell^0_t(\pi_0(q))\right)+\sum_{k=1}^K \sum_{t} \big[p_t^{k}\cdot \ell_t^k - \ell_t^k(\pi_{k}(q)) \big]\\
~=~\sum_{k=0}^K \sum_{t\in\Real} \big[ p_t^{k}\cdot \ell_t^k - \ell_t^k(\pi_{k}(q)) \big] \enspace.
\end{align*}
We obtain that the algorithm's regret is the cumulative sum of the $K+1$ regrets obtained by each algorithm that observes losses $\ell_t^{k}$. Since each level $p_t^k$ is updated according to the \PROD algorithm we get via \Cref{lem:exp}, for $k\ge 0$:
\begin{align*}\E\left[\sum_{t\in\Real}\left( p_{t}^{k}\cdot \ell_t^k-\ell^k_t(\pi_k(q))\right)\right]
&\le \frac{\log |Q_{I_k}(\cal H)|}{\eta_k} + \eta_k \E\left[\sum_{t\in\Real} \left(\ell^{k}_t(\pi_k(q))\right)^2\right]\\
&\le \frac{d\log (\beta_k T)}{\eta_k} + \eta_k \E\left[\sum_{t\in\Real} \left(\ell^{k}_t(\pi_k(q))\right)^2\right]&\textrm{\Cref{eq:logQ}} \labelthis{eq:kregret}
\end{align*}
Applying the above bound for $k=0$ and using $|\ell_t^0|\le 1$, we have
\begin{align*}
\E\Big[\sum_{t\in\Real}\left( p_{t}^{0}\cdot \ell_t^0-\ell^0_t(\pi_0(q))\right)\Big]
~\le~ \frac{d\log (\beta_0 T)}{\eta_0} +\eta_0 T ~\le~ 2\sqrt{dT\log(\beta_0 T)}+4d\log(\beta_0 T).
\end{align*}

For $k>0$ we improve the bound on the second-order term. Since $q$ is the terminal predictor induced by the fixed hypothesis $h$, and the parent maps restrict traces, we have
\[
        \pi_k(q)=Q(S_{I_k}^{h},\cdot),
        \qquad
        \pi_{k-1}(q)=Q(S_{I_{k-1}}^{h},\cdot).
\]
Hence,
%For $k>0$ we improve the bound on the second order term. Set $h_q\in \calH$ to be such that $\pi_k(q)= Q(S^{h_q}_{I_k},\cdot)$, then by definition we have that $\pi_{k-1}(q)= Q(S^{h_{q}}_{I_{k-1}},\cdot)$, hence
%
\begin{align*}
\E\left[\sum_{t\in \Real}\left(\ell^{k}_t(\pi_k(q))\right)^2\right]
&\le
\E\left[\sum_{t\in \Real}
\mathbf{1}\{\pi_k(q)(x_t)\ne \pi_{k-1}(q)(x_t)\}
\right]\\
&\le
\E\left[\sum_{t\in \Real}
\mathbf{1}\{\pi_k(q)(x_t)\ne h(x_t)\}
+
\mathbf{1}\{\pi_{k-1}(q)(x_t)\ne h(x_t)\}
\right].
\end{align*}
Because $\Sample$ is uniform and each $I_k$ is drawn uniformly from $I_{k+1}$, every $I_k$ is marginally a uniformly random subset of $[T]$ of its prescribed size. Also, since $I_k\subseteq \Sample$, we have $\Real\subseteq[T]\setminus I_k$. Therefore \Cref{cor:oinclusion} gives
\begin{align*}
\E\left[\sum_{t\in \Real}\left(\ell^{k}_t(\pi_k(q))\right)^2\right]
&\le \frac{2 dT}{|I_{k}|}+ \frac{2 dT}{|I_{k-1}|} \\
&\le \frac{4 dT}{|I_{k-1}|} \quad = \quad \frac{4e}{\beta_{k-1}} \qquad \textrm{ using }\beta_{k-1} = \frac{e|I_{k-1}|}{dT}.
\end{align*}

Plugging the inequality in \Cref{eq:kregret}, and considering our choice of $\eta_k$ yields the result. As in the $k=0$ case, if the square-root learning rate term in the definition of $\eta_k$ exceeds $1/2$, then the same bound holds after increasing constants, because this can only happen when the log-size term already dominates the variance term.
\end{proof}

\subsection{Proof of the upper bound in \Cref{thm:binary-upper}}\label{prf:binary-upper}
%\subsection{Proof of \Cref{thm:binary-upper}}\label{prf:binary-upper}

We are now ready to prove the main result of the Binary classification case. 
If $d/p\ge T$, the trivial bound $L_{\Real}(\calA)\le T$ proves the theorem. Hence assume $pT\ge d$. For readability, assume $pT/d$ is a power of two and set $K=\log_2(pT/d)$; rounding the levels changes only constants. We run \Cref{alg:hier_prod} with
\[
        \beta_k=\frac{e2^k}{T},
        \qquad k=0,\ldots,K.
\]

First, apply \Cref{cor:oinclusion} with the fixed comparator $h^\star$. Since $|\Sample|=pT$, let $q=Q(S_{\Sample}^{h^\star},\cdot)\in Q_{\Sample}(\calH)$. Since $q$ is binary-valued, $|q(x_t)-y_t|=\mathbf{1}\{q(x_t)\ne y_t\}$. Hence,
\[
\E\Big[\sum_{t\in \Real}\mathbf{1}\{q(x_t)\ne h^\star(x_t)\}\Big]
    \le O(d/p).
\]
Therefore,
\[
\E\Big[\sum_{t\in \Real}\mathbf{1}\{q(x_t)\ne y_t\}\Big]
\le
\E L_{\Real}(h^\star)+O(d/p).
\]

Next, since the algorithm predicts $1$ with conditional probability $r_t$, we have by \Cref{lem:hier_prod_regret} that
\begin{align*}
\E\Big[
L_{\Real}(\calA)-\sum_{t\in\Real}|q(x_t)-y_t|
\Big] 
%Next, by \Cref{lem:hier_prod_regret} as our algorithm predicts $\hat y_t$ \begin{align*} \E\left[\sum_{t\in \Real}\mathbf{1}[\hat y_t\ne y_t]- \mathbf{1}[q(x_t)\ne y_t]\right]
&\le  2\sqrt{dT \log (\beta_0 T)} + 4\sum_{k=1}^K \sqrt{\frac{4ed\log (\beta_k T)}{\beta_{k-1}}} + 5d\sum_{k=0}^K \log (\beta_k T)\\
& \le 4\sqrt{dT} +4\sum_{k=1}^K \sqrt{T\frac{4d(k+2)}{2^{k}}}+10dK^2
 ~=~ O\left(\sqrt{dT}\right)+10dK^2.
\end{align*}
Notice that
\[ 10 dK^2 ~\le~ 10d\log^2 \left(\frac{epT}{d}\right)~\le~ 10 d\log^2\frac{eT}{d} ~=~ O\left(d\sqrt{\frac{T}{d}}\right) ~=~ O(\sqrt{dT}),\]
where we used the inequality $\log^2(er)\le C\sqrt r$ for $r\ge1$, with $r=T/d$.
Adding these two error terms gives the final result.

\begin{remark}[Comparator on the unrevealed sequence]\label{rem:real-comparator}
The binary upper bound also implies the same rate against the stronger comparator
\(\min_{h\in H}L_{\mathrm{Real}}(h)\). Let
\(h_T \in \arg\min_{h\in H}L_{[T]}(h)\) and
\(h_R \in \arg\min_{h\in H}L_{\mathrm{Real}}(h)\). For the binary loss class
$\mathcal L_H=\{t\mapsto \mathbf 1[h(x_t)\neq y_t]:h\in H\}$,
which has VC dimension at most \(d\), standard transductive VC uniform convergence gives
\[
\mathbb E_{\mathrm{Sample}}\sup_{h\in H}
\left|L_{\mathrm{Real}}(h)-(1-p)L_{[T]}(h)\right|
    ~\le~ C\sqrt{dT}.
\]
Therefore,
\[
\mathbb E\!\left[L_{\mathrm{Real}}(h_T)-L_{\mathrm{Real}}(h_R)\right]
~\le~
2\,\mathbb E\sup_{h\in H}
\left|L_{\mathrm{Real}}(h)-(1-p)L_{[T]}(h)\right| ~\le~ C'\sqrt{dT}.
\]
Combining this with \Cref{thm:binary-upper} yields
\[
\mathbb E\!\left[
L_{\mathrm{Real}}(A)-\min_{h\in H}L_{\mathrm{Real}}(h)
\right]
\le
C''\min\left\{T,\frac d p+\sqrt{dT}\right\}.
\]
The matching lower bound of \Cref{thm:binary-upper} also applies to this stronger comparator, since
\(\min_{h\in H}L_{\mathrm{Real}}(h)\le L_{\mathrm{Real}}(h^\star)\) pointwise.
\end{remark}

\subsection{Binary lower bound in \Cref{thm:binary-upper}}\label{sec:binary-lower}
We give the standard two-construction lower bound. By Yao's principle, it
suffices to give distributions over fixed labeled sequences against which
every deterministic learner has large expected regret.

For the $\sqrt{dT}$ term, let $d_0=\min\{d,T\}$ and
$n=\lfloor T/d_0\rfloor$. Use the first $d_0n$ rounds and fill any
remaining rounds with a dummy example whose label is known and on which
all hypotheses agree. Take a class that shatters
$u_1,\ldots,u_{d_0}$, present each $u_i$ for $n$ rounds, and label the
copies by independent fair signs $\sigma_{i,1},\ldots,\sigma_{i,n}$. If
$R_i$ is the set of unrevealed copies of $u_i$, then conditional on the
preview and the past online labels, each not-yet-revealed label is still
fair, so the learner's expected loss on $R_i$ is $|R_i|/2$. The
full-sequence optimum predicts the majority sign
       $ b_i=\operatorname{sign}\Big(\sum_{j=1}^n\sigma_{i,j}\Big)$,
with deterministic tie-breaking. By exchangeability,
       $ \bbE[b_i\sigma_{i,j}]
        =
        \frac1n\,
        \bbE\big|\sum_{j=1}^n\sigma_{i,j}\big|
        =
        \Theta(n^{-1/2})$.
Since the preview is independent of the labels and
$\bbE|R_i|=(1-p)n=\Theta(n)$, the majority classifier has expected
advantage $\Omega(\sqrt n)$ on the real copies of $u_i$. Summing over $i$ gives     $\Omega(d_0\sqrt n)
        =
        \Omega(\min\{T,\sqrt{dT}\})$.

For the $d/p$ term, it remains to consider $d\le T$; set
$\ell=\lfloor T/d\rfloor$.
 Again fill unused rounds with dummy examples. Take $d$ disjoint copies of the threshold class,
        $h_{\theta_1,\ldots,\theta_d}(i,x)
        =
        {\bf 1}\{x\le \theta_i\}$,
which has VC dimension $d$. On each copy, use the depth-$\ell$
Littlestone binary-search tree for thresholds \cite{Littlestone1988}:
choose a random root-to-leaf path and serve its nodes in order. The
sequence is realizable. For one copy, let $J$ be the deepest previewed
level, with $J=0$ if the copy has no previewed node. The deepest previewed
instance reveals the path prefix, but the suffix after level $J$ remains
independent fair path bits, so any learner makes $\Omega(\ell-J)$
expected mistakes on that copy. Moreover,
        $\bbE[\ell-J]
        =
        \sum_{r=1}^{\ell}\bbP\{\ell-J\ge r\}$.
Under exact-size preview sampling, each fixed node is previewed with
probability $p$. Hence, by the union bound, for
$r\le\min\{\ell,1/(2p)\}$ the probability that no preview point lies in
the last $r$ levels is at least $1-pr\ge1/2$. Thus
        $\bbE[\ell-J]
        =
        \Omega(\min\{\ell,1/p\})$.
Summing over the $d$ copies gives
        $\Omega(d\min\{\ell,1/p\})
        =
        \Omega(\min\{T,d/p\})$.

Taking the better of the two constructions and using
$\max\{\min(T,a),\min(T,b)\}\ge \frac12\min\{T,a+b\}$ with
$a=d/p$ and $b=\sqrt{dT}$ gives the claimed
$\Omega(\min\{T,d/p+\sqrt{dT}\})$ lower bound.

\section{Multiclass Classification}\label{sec:multiclass}

We now prove the multiclass theorem.  The main new issue is not the online protocol; it is the label space.  In the binary proof, we constructed a set of predictors by considering the possible labelings, and Sauer--Shelah bounded the number of possible predictors we needed. When $\calY$ is infinite, or simply very large, even for simple classes this strategy is no longer scalable.  For example, consider the class
        $\calH_{\rm const}=\{h_y:x\mapsto y \mid y\in\calY\}$
of classifiers that ignore the instance and always output a fixed label.  This class has Natarajan and DS dimension one, but its trace on any nonempty preview contains one labeling for every $y\in\calY$.
Thus a trace-enumeration argument would introduce an unacceptable dependence on $|\calY|$.

The offline multiclass theory resolves exactly this kind of difficulty by reducing the label space locally.  The reduction has three phases.  First, DS dimension is used to build a finite proxy class that protects the correct predictions of the comparator.  Second, a finite proxy class is converted into a short menu of candidate labels. Finally, once the correct labels are restricted to this menu, Natarajan dimension controls the remaining finite-label problem.  Our contribution in this section is to place this reduction inside the preview/online protocol: the preview constructs the finite objects, and the online labels are used by multiplicative weights.

We prove \Cref{thm:multiclass} after stating the multiclass ingredients used in the proof.

\subsection{Tools from Offline Multiclass Classification} \label{sec:multiclassTools}

 The following three ingredients are standard consequences of \cite{CohenErezHannekeKorenMansourMoranZhang2025}, with Pabbaraju's DS-density theorem~\cite{Pabbaraju2026} plugged into the first step. 
 Since they are stated as ordinary distributional statements, we apply them to the uniform distribution over the fixed length-$T$ indexed sequence.
Formally, let
\begin{align} \label{defn:PT}
        P_T=\frac1T\sum_{t=1}^T \delta_{((t,x_t),y_t)}
\end{align}
be the uniform distribution on the indexed labeled sequence.
We view every $h\in\calH$ as a classifier on indexed examples by writing
     $   h(t,x_t)=h(x_t)$.
The indexing is only bookkeeping: it lets the offline multiclass lemmas apply to the finite population of examples that appears in the preview problem.
 
\paragraph{Tool 1: DS correct-region proxy.}
The first ingredient says that a preview can produce a finite proxy class which protects the correct predictions of any fixed comparator.  The guarantee is one-sided: the proxy only needs to imitate the comparator on points where the comparator is correct.

\begin{lemma}[DS correct-region cover]\label{lem:ds-proxy}
Let $h^\star\in\calH$ be fixed independently of an i.i.d. sample of size $n$ from $P_T$.  There is a procedure which, from this sample, outputs a finite class $\calF_1\subseteq\calY^{[T]\times\calX}$ such that, with probability at least $1-\delta$, we have
        $\log |\calF_1|\le \tO(D),$
and there exists $f^\star\in\calF_1$ satisfying
\[
        P_T\{h^\star(X)=Y,\ f^\star(X)\ne Y\}
        \le
        \tO\!\left(\frac{D+\log(1/\delta)}{n}\right).
\]
\end{lemma}

\begin{proof}[Proof sketch]
This is the correct-region cover construction of \cite[Theorem 3.2]{CohenErezHannekeKorenMansourMoranZhang2025}, applied with the realizable compression/list-learning component supplied by their Proposition 3.3.  Pabbaraju's DS-density theorem~\cite{Pabbaraju2026} gives the optimal DS-controlled density/compression parameter for this component.  Substituting that bound into the \cite{CohenErezHannekeKorenMansourMoranZhang2025} construction gives a proxy class of logarithmic size $\tO(D)$ and the displayed one-sided error bound.  The statement is pointwise in the fixed comparator, which is what we need because $h^\star$ is chosen independently of the preview.
\end{proof}

\paragraph{Tool 2: Menus from finite classes.}
The second ingredient converts a finite proxy class (obtained from \Cref{lem:ds-proxy}) into a short menu of labels.  The menu is local: for each indexed example $X$, it is a small set $\mu(X)\subseteq\calY$.  The guarantee is again one-sided: when a proxy classifier is correct, its label is usually included in the menu.

\begin{lemma}[Finite-class menu reduction]\label{lem:finite-menu}
Let $\calF$ be a finite class of classifiers from $\Omega$ to $\calY$, and let $P$ be a distribution over $\Omega\times\calY$. Given $n$ labeled samples from $P$, there is a procedure returning a menu $\mu:\Omega\to 2^\calY$ such that
        $\log\max_{\omega\in\Omega}|\mu(\omega)|\le O(\log n)$,
and for every fixed $f^\star\in\calF$ chosen independently of the sample used to build $\mu$, with probability at least $1-\delta$,
\[
        P\{f^\star(X)=Y,\ f^\star(X)\notin\mu(X)\}
        \le
        C\frac{\log|\calF|+\log(1/\delta)}{n}.
\]
\end{lemma}

\begin{proof}[Proof sketch]
\cite[Theorem 3.4]{CohenErezHannekeKorenMansourMoranZhang2025} gives a multiplicative-weights procedure for a finite class $\calF_1$.  The procedure outputs a list of at most $n$ classifiers $g_1,\ldots,g_r\in\calF_1$ and uses the induced menu
        $\mu(X)=\{g_1(X),\ldots,g_r(X)\}$.
Hence $\max_X|\mu(X)|\le r\le n$, which gives the menu-size bound. 
The same theorem gives, for any fixed comparator $f^\star\in\calF$ chosen independently of the sample used to build the menu, the one-sided guarantee
\[
        P\{f^\star(X)=Y,\ f^\star(X)\notin\mu(X)\}
        ~\le~
        C\frac{\log|\calF|+\log(1/\delta)}{n}.  \qedhere
\]
\end{proof}

\paragraph{Tool 3: Natarajan cover inside a menu.}
The final ingredient uses the menu to reduce the effective label space.  Once the correct label is known to lie in a menu of size $M$, the relevant finite-label complexity is controlled by Natarajan dimension, up to a $\log M$ factor.

\begin{lemma}[Natarajan cover inside a menu]\label{lem:nat-menu-cover}
Fix a menu $\mu:[T]\times\calX\to2^{\calY}$ with $\max_X|\mu(X)|\le M$.  Let $h\in\calH$ be fixed independently of an i.i.d. sample of size $n$ from $P_T$.  There is a procedure which, from this sample and the menu $\mu$, outputs a finite class $\calF_2\subseteq\calY^{[T]\times\calX}$ such that, with probability at least $1-\delta$, we have
        $\log |\calF_2|\le \tO(N\log M)$,
and there exists $g_h\in\calF_2$ satisfying
\[
        P_T\{h(X)=Y,\ h(X)\in\mu(X),\ g_h(X)\ne Y\}
        \le
        \tO\!\left(\frac{N\log M+\log(1/\delta)}{n}\right).
\]
\end{lemma}

\begin{proof}[Proof sketch]
This is the menu-restricted cover theorem of \cite[Theorem 3.5]{CohenErezHannekeKorenMansourMoranZhang2025}.  Their Proposition 3.6 gives the underlying compression/list-learning bound with size controlled by $d_{\rm Nat}(\calH)\log M$, because the menu restricts the effective label space to $M$ labels at each point.  Applying their construction to the fixed menu $\mu$ yields a finite class of logarithmic size $\tO(N\log M)$ and the displayed one-sided error guarantee for the fixed comparator $h$.
\end{proof}

\subsection{Proof of the Multiclass Upper Bound in \Cref{thm:multiclass}}

Since the multiclass offline tools in \Cref{sec:multiclassTools} are for i.i.d. arrivals, we define $P_T$ to be the uniform distribution over the indexed labeled sequence, as in \eqref{defn:PT}.  Choose three sample sizes
\[
        n_1=n_2=n_3= \Big\lfloor\frac{m}{3} \Big\rfloor .
\]
The following \Cref{lem:exact-preview-coupling} shows how to obtain three independent i.i.d. samples from $P_T$, of sizes $n_1,n_2,n_3$, from the preview of size $m$.

\begin{lemma}[Exact preview can supply i.i.d. sub-samples]\label{lem:exact-preview-coupling}
%Fix integers $n_1,\ldots,n_k$ with $n_1+\cdots+n_k\le m$.  
There is a randomized procedure which, given a uniformly random preview $\Sample\subseteq[T]$ of size $m$, outputs index sequences
\[
        I^{(a)}_1,\ldots,I^{(a)}_{n_a}\in\Sample
        \quad \text{for } a\in\{1,2,3\},
\]
such that, unconditionally, the full collection of indices 
       $\bigl(I^{(a)}_j\bigr)_{a\in[3],\,j\in[n_a]}$
is distributed as independent uniform samples from $[T]$.
\end{lemma}

\begin{proof}
It suffices to describe a coupling.  First draw $n=n_1+n_2+n_3$ independent uniform indices $I_1,\ldots,I_n$ from $[T]$.  Let $D_I\subseteq[T]$ be the set of distinct indices appearing among them.  Since $n\le m$, we can then choose $\Sample$ uniformly among all $m$-subsets of $[T]$ containing $D_I$.  Finally, split the sequence $I_1,\ldots,I_n$ into consecutive blocks of lengths $n_1,
\ldots,n_3$.

By construction, all sampled indices lie in $\Sample$ and the blocks are independent uniform samples from $[T]$.  The marginal distribution of $\Sample$ is uniform over all $m$-subsets: the construction is invariant under every permutation of $[T]$, and therefore assigns the same probability to every $m$-subset.  Thus this coupling has exactly the same preview marginal as the model.  Equivalently, after observing $\Sample$, the learner can sample the block indices from the corresponding conditional distribution; this conditional distribution depends only on $\Sample$, $T$, and the block lengths, and all queried labels are available because the indices lie in $\Sample$.
\end{proof}

We will now use the three sampled index sequences $I^{(1)},I^{(2)},I^{(3)}$ obtained from \Cref{lem:exact-preview-coupling}  for the three stages below.  Let $\delta=T^{-4}$.

\paragraph{Stage 1: build a DS proxy class.}
Apply \Cref{lem:ds-proxy} to the first sampled index sequence $I^{(1)}=(I^{(1)}_1,\ldots,I^{(1)}_{n_1})$ with the fixed comparator $h^\star$. With probability at least $1-\delta$, we obtain a class $\calF_1$ satisfying
        $\log|\calF_1|\le \tO(D)$
and a proxy $f^\star\in\calF_1$ such that
\[
        \big|\{t\in[T]: h^\star(x_t)=y_t,
        \ f^\star(t,x_t)\ne y_t\}\big|
        ~\le~
        T\cdot \tO\!\left(\frac{D+\log(1/\delta)}{n_1}\right) 
        ~\le~ \tO\!\left(\frac Dp\right),
\]
where the last inequality uses $n_1=\Theta(pT)$ and absorbs logarithms into $\tO(\cdot)$.

\paragraph{A useful DS-only stopping point.}
If one is content with a bound depending only on DS dimension (i.e. not the improved guarantee depending on Natarajan dimension), this already suffices.  Indeed, after constructing $\calF_1$, one can run multiplicative weights directly over $\calF_1$ during the online phase.  The proxy error contributes $\tO(D/p)$ and the online aggregation contributes
        $\sqrt{T\log |\calF_1|}=\tO(\sqrt{DT})$.
This gives
\[
        \bbE\big[L_{\Real}(\calA)-L_{\Real}(h^\star)\big]
        ~\le~
        \tO\!\left(\frac Dp+\sqrt{DT}\right).
\]
The remaining two ingredients are used only to improve the second term from $\sqrt{DT}$ to $\sqrt{NT}$.

\paragraph{Stage 2: build a menu.}
Apply \Cref{lem:finite-menu} to $\calF_1$ using the second sampled index sequence $I^{(2)}=(I^{(2)}_1,\ldots,I^{(2)}_{n_2})$. Conditional on the outcome of Stage 1, fix one proxy $f^\star\in\calF_1$ satisfying the Stage 1 guarantee, using an arbitrary deterministic tie-breaking rule if there are several. Since the second sampled sequence is independent of the first, the pointwise guarantee of \Cref{lem:finite-menu} applies to this fixed $f^\star$. Thus, with probability at least $1-\delta$, the resulting menu $\mu$ satisfies
        $\log M:=\log\max_X|\mu(X)|\le O(\log T)$
and
\[
        P_T\{f^\star(X)=Y,\ f^\star(X)\notin\mu(X)\}
        ~\le~
        C\frac{\log|\calF_1|+\log(1/\delta)}{n_2}.
\]
Since $P_T$ is uniform over the fixed indexed sequence, this implies
\[
        \big|\{t\in[T]: f^\star(t,x_t)=y_t,\
        f^\star(t,x_t)\notin\mu(t,x_t)\}\big|
        ~\le~
        \tO\!\left(\frac Dp\right).
\]
Combining this with the Stage 1 error, we get that the menu contains the correct label of $h^\star$ on all but $\tO(D/p)$ of the points where $h^\star$ is correct:
\[
        \big|\{t\in[T]: h^\star(x_t)=y_t,\
        h^\star(x_t)\notin\mu(t,x_t)\}\big|
        ~\le~
        \tO\!\left(\frac Dp\right).
\]
Indeed, if $h^\star(x_t)=y_t$ but $h^\star(x_t)\notin\mu(t,x_t)$, then either $f^\star(t,x_t)\ne y_t$, or $f^\star(t,x_t)=y_t$ and that label is missing from the menu.

\paragraph{Stage 3: build a Natarajan-size expert class.}
Condition on the menu $\mu$ produced in Stage 2. 
Since the third sampled sequence $I^{(3)}=(I^{(3)}_1,\ldots,I^{(3)}_{n_3})$
is independent of the first two sampled sequences,  we may apply \Cref{lem:nat-menu-cover} with this fixed menu and the fixed comparator $h^\star$.  With probability at least $1-\delta$, we obtain a class $\calF_2$ with
\[
        \log|\calF_2|~\le~ \tO(N\log M) ~=~ \tO(N)
\]
and a classifier $g^\star\in\calF_2$ such that
\[
        \big|\{t\in[T]: h^\star(x_t)=y_t,
        \ h^\star(x_t)\in\mu(t,x_t),
        \ g^\star(t,x_t)\ne y_t\}\big|
        ~\le~
        \tO\!\left(\frac{N\log M}{p}\right)
        ~\le~
        \tO\!\left(\frac Dp\right),
\]
where the final inequality uses $N\le D$ and absorbs the $\log M=O(\log T)$ factor into $\tO(\cdot)$.

We now compare $g^\star$ with $h^\star$.  On a point where $h^\star$ is wrong, the excess loss of $g^\star$ over $h^\star$ is at most zero, since both losses are in $\{0,1\}$ and $h^\star$ already pays one unit.
 On a point where $h^\star$ is correct, $g^\star$ can be wrong only if either the correct label of $h^\star$ is missing from the menu, or it is in the menu but $g^\star$ fails to predict it.  The two bounds above therefore imply
\[
        L_{\Real}(g^\star)-L_{\Real}(h^\star)
        ~\le~
        \tO\!\left(\frac Dp\right).
\]
We used counts over $[T]$ to bound counts over $\Real$.

\paragraph{Stage 4: online aggregation.}
Finally, during the online phase, run multiplicative weights over the finite expert class $\calF_2$.  On round $t\in\Real$, expert $g\in\calF_2$ predicts $g(t,x_t)$, and the algorithm predicts by sampling an expert from the current multiplicative-weights distribution.  The expected loss is the weighted average of the expert losses.  
Conditioned on the good event from Stage 3, $g^\star$ is a fixed expert in $\calF_2$. Applying the standard full-information experts bound to this expert and using $\log|\calF_2|\le\tO(N)$ gives
\[
        \bbE_{\calA}\!\left[L_{\Real}(\calA)-L_{\Real}(g^\star)\right]
        \le \tO(\sqrt{NT}).
\]
Combining with the comparison between $g^\star$ and $h^\star$, we get on the intersection of the three good events
\[
        \bbE_{\calA}\!\left[
        L_{\Real}(\calA)-L_{\Real}(h^\star)
        \right]
        \le
        \tO\!\left(\frac Dp+\sqrt{NT}\right).
\]
The three imported tools fail with probability at most $3\delta$, and the loss difference is always bounded by $T$.  With $\delta=T^{-4}$, the failure contribution to expectation is negligible.  This proves \Cref{thm:multiclass}.

\paragraph{Acknowledgments.}
The authors are thankful to the organizers of the PhD School on Intersections of Algorithms and Machine Learning Theory, where this project began.
The authors used GPT-5.5 Pro for assistance with writing and  to explore
proof strategies. The authors have checked the arguments and take full responsibility for all content of the paper.

{%\small
\bibliographystyle{alpha}
\bibliography{bib}
}

\end{document}